\documentclass{article}

\usepackage{microtype}
\usepackage{graphicx}
\usepackage{subfig}
\usepackage{booktabs} 

\usepackage{hyperref}

\usepackage{floatrow}
\usepackage{url}
\usepackage{amsfonts}       
\usepackage{nicefrac}       
\usepackage{xcolor}
\usepackage{amsmath}
\usepackage{amssymb}
\usepackage{amsthm}
\usepackage{multirow}
\usepackage{tabularx}

\DeclareMathOperator*{\argmin}{arg\,min}

\newfloatcommand{capbtabbox}{table}[][\FBwidth]

\makeatletter
\newtheorem*{rep@theorem}{\rep@title}
\newcommand{\newreptheorem}[2]{%
\newenvironment{rep#1}[1]{%
 \def\rep@title{#2 \ref{##1}}%
 \begin{rep@theorem}}%
 {\end{rep@theorem}}}
\makeatother

\newtheorem{lemma}{Lemma}
\newtheorem{proposition}{Proposition}
\newtheorem{theorem}{Theorem}

\newtheorem{corollary}{Corollary}
\newtheorem{assumption}{Assumption}
\newreptheorem{lemma}{Lemma}
\newreptheorem{proposition}{Proposition}
\newreptheorem{theorem}{Theorem}
\newreptheorem{defn}{Definition}
\newreptheorem{conjecture}{Conjecture}
\newreptheorem{corollary}{Corollary}
\newreptheorem{assumption}{Assumption}

\newcommand\R{\ensuremath{\mathbb{R}}} 

\newcommand\refsec[1]{Section~\ref{sec:#1}}

\newcommand\reffig[1]{Figure~\ref{fig:#1}}

\newcommand\reftab[1]{Table~\ref{tab:#1}}
\newcommand\refapp[1]{Appendix~\ref{sec:#1}}

\ifthenelse{\isundefined{\definition}}{\newtheorem{definition}{Definition}}{}
\ifthenelse{\isundefined{\assumption}}{\newtheorem{assumption}{Assumption}}{}
\ifthenelse{\isundefined{\hypothesis}}{\newtheorem{hypothesis}{Hypothesis}}{}
\ifthenelse{\isundefined{\proposition}}{\newtheorem{proposition}{Proposition}}{}
\ifthenelse{\isundefined{\theorem}}{\newtheorem{theorem}{Theorem}}{}
\ifthenelse{\isundefined{\lemma}}{\newtheorem{lemma}{Lemma}}{}
\ifthenelse{\isundefined{\corollary}}{\newtheorem{corollary}{Corollary}}{}
\ifthenelse{\isundefined{\alg}}{\newtheorem{alg}{Algorithm}}{}
\ifthenelse{\isundefined{\example}}{\newtheorem{example}{Example}}{}

\usepackage[accepted]{icml2020}

\icmltitlerunning{Concept Bottleneck Models}

\begin{document}

\twocolumn[
\icmltitle{Concept Bottleneck Models}



\icmlsetsymbol{equal}{*}

\begin{icmlauthorlist}
\icmlauthor{Pang Wei Koh}{equal,stan}
\icmlauthor{Thao Nguyen}{equal,stan,goo}
\icmlauthor{Yew Siang Tang}{equal,stan}\\
\icmlauthor{Stephen Mussmann}{stan}
\icmlauthor{Emma Pierson}{stan}
\icmlauthor{Been Kim}{goo}
\icmlauthor{Percy Liang}{stan}
\end{icmlauthorlist}

\icmlaffiliation{stan}{Stanford University}
\icmlaffiliation{goo}{Google Research}

\icmlcorrespondingauthor{Pang Wei Koh}{pangwei@cs.stanford.edu}
\icmlcorrespondingauthor{Been Kim}{beenkim@google.com}
\icmlcorrespondingauthor{Percy Liang}{pliang@cs.stanford.edu}

\icmlkeywords{Machine Learning, ICML}

\vskip 0.3in
]



\printAffiliationsAndNotice{\icmlEqualContribution} 

\begin{abstract}
  We seek to learn models that we can interact with using high-level concepts: if the model did not think there was a bone spur in the x-ray, would it still predict severe arthritis? State-of-the-art models today do not typically support the manipulation of concepts like ``the existence of bone spurs'', as they are trained end-to-end to go directly from raw input (e.g., pixels) to output (e.g., arthritis severity). We revisit the classic idea of first predicting concepts that are provided at training time, and then using these concepts to predict the label. By construction, we can intervene on these \emph{concept bottleneck models} by editing their predicted concept values and propagating these changes to the final prediction. On x-ray grading and bird identification, concept bottleneck models achieve competitive accuracy with standard end-to-end models, while enabling interpretation in terms of high-level clinical concepts (``bone spurs'') or bird attributes (``wing color''). These models also allow for richer human-model interaction: accuracy improves significantly if we can correct model mistakes on concepts at test time.

\end{abstract}
\vspace{-5mm}
\section{Introduction}\label{sec:intro}

Suppose that a radiologist is collaborating with a machine learning model to grade the severity of knee osteoarthritis. She might ask why the model made its prediction---did it deem the space between the knee joints too narrow? Or she might seek to intervene on the model---if she told it that the x-ray showed a bone spur, would its prediction change?

State-of-the-art models today do not typically support such queries: they are end-to-end models that go directly from raw input $x$ (e.g., pixels) to target $y$ (e.g., arthritis severity), and we cannot easily interact with them using the same high-level concepts that practitioners reason with, like ``joint space narrowing'' or ``bone spurs''.

We approach this problem by revisiting the simple idea of first predicting an intermediate set of human-specified concepts $c$
like ``joint space narrowing'' and ``bone spurs'',
then using $c$ to predict the target $y$. In this paper, we refer to such models as \emph{concept bottleneck models}.
These models are trained on data points $(x, c, y)$, where the input $x$ is annotated with both concepts $c$ and target $y$.
At test time, they take in an input $x$, predict concepts $\hat{c}$, and then use those concepts to predict the target $\hat{y}$ (\reffig{intro}).

\begin{figure}[t]
  \begin{center}
    \includegraphics[width=\linewidth]{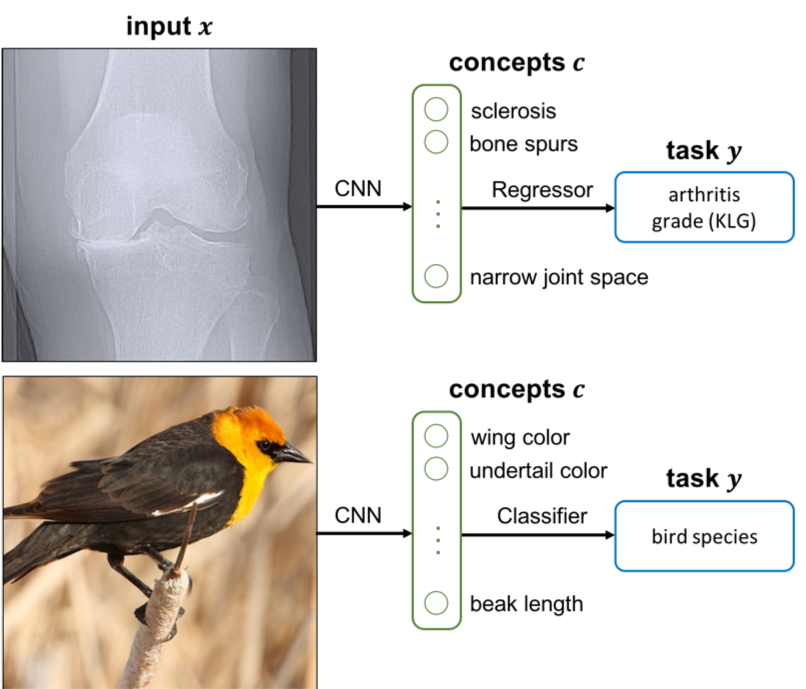}
    \caption{
      We study concept bottleneck models that first predict an intermediate set of human-specified concepts $c$,
      then use $c$ to predict the final output $y$. We illustrate the two applications we consider:
      knee x-ray grading and bird identification.
      \vspace{-5mm}
    }
    \vspace{-5mm}
    \label{fig:intro}
  \end{center}
\end{figure}

Earlier versions of concept bottleneck models were overtaken in predictive accuracy by end-to-end neural networks (e.g., \citet{kumar2009attribute} for face recognition and \citet{lampert2009learning} for animal identification),
leading to a perceived tradeoff between accuracy and interpretability in terms of concepts.
Recently, concept bottleneck models have started to re-emerge as targeted tools for solving particular tasks \citep{de2018clinically, yi2018neural, bucher2018semantic, losch2019interpretability, chen2020concept}.

In this paper, we propose a straightforward method for turning any end-to-end neural network into a concept bottleneck model, given concept annotations at training time: we simply resize one of the layers to match the number of concepts provided, and add an intermediate loss that encourages the neurons in that layer to align component-wise to the provided concepts.
We show that concept bottleneck models trained in this manner can achieve task accuracies competitive with or even higher than standard models.
We emphasize that concept annotations are not needed at test time; the model predicts the concepts, then uses the predicted concepts to make a final prediction.

Importantly---and unlike standard models---%
these bottleneck models allow us to intervene on concepts by editing the concept predictions $\hat{c}$ and propagating those changes to the target prediction $\hat{y}$.
Interventions enable richer human-model interaction: e.g., if the radiologist realizes that what the model thinks is a bone spur is actually an artifact, she can update the model's prediction by directly changing the corresponding value of $\hat{c}$.
When we simulate
this injection of human knowledge by partially correcting concept mistakes that the model makes at test time,
we find that accuracy improves substantially beyond that of a standard model.

Interventions also make concept bottleneck models interpretable in terms of high-level concepts: by manipulating concepts $\hat{c}$ and observing the model's response, we can obtain counterfactual explanations like ``if the model did not think the joint space was too narrow for this patient, then it would not have predicted severe arthritis''.
In contrast, prior work on explaining end-to-end models in terms of high-level concepts has been restricted to post-hoc interpretation of already-trained models: e.g., predicting concepts from hidden layers \citep{kim2018interpretability} or measuring the correlation of individual neurons with concepts \citep{bau2017network}.

The validity of interventions on a model depends on the alignment between its predicted concepts $\hat{c}$ and the true concepts $c$.
We can estimate this alignment by measuring the model's concept accuracy on a held-out validation set \citep{fong2017interpretable}.\footnote{%
  With the usual caveats of measuring accuracy: in practice, the validation set might be skewed such that models that learn spurious correlations can still achieve high concept accuracy.
}
A model with perfect concept accuracy across all possible inputs makes predictions $\hat{c}$ that align
with the true concepts $c$.
Conversely, if a model has low concept accuracy, then the model's predictions $\hat{c}$ need not match with the true concepts, and we would not expect interventions to lead to meaningful results.

\textbf{Contributions.}
We systematically study variants of concept bottleneck models and contrast them with standard end-to-end models
in different settings, with a focus on the previously-unexplored ability of concept bottleneck models to support concept interventions.
Our goal is to characterize concept bottleneck models more fully:
Is there a tradeoff between task accuracy and concept interpretability?
Do interventions at test time help model accuracy, and is concept accuracy a good indicator of the ability to effectively intervene?
Do different ways of training bottleneck models lead to significantly different outcomes in intervention?

We evaluate concept bottleneck models on the two applications in \reffig{intro}:
the osteoarthritis grading task \citep{nevitt2006osteoarthritis} and a fine-grained bird species identification task \citep{wah2011cub}.
On these, we show that bottleneck models are comparable to standard end-to-end models
while also attaining high concept accuracies.
In contrast, the concepts cannot be predicted with high accuracy
from linear combinations of neurons in a standard black-box model,
making it difficult to do post-hoc interpretation in terms of concepts like in \citet{kim2018interpretability}.
We demonstrate that we can substantially improve model accuracy by intervening on these bottleneck models at test time to correct model mistakes on concepts,
and we show that different methods of training bottleneck models lead to different trade-offs between task accuracy with and without interventions.
Finally, we show that bottleneck models guided to learn the right concepts can also be more robust to covariate shifts.

\section{Related work}\label{sec:related}

\textbf{Concept bottleneck models.} Models that bottleneck on human-specified
concepts---where the model first predicts the concepts, then uses only those predicted concepts to make a final prediction---%
have been previously used for specific applications \citep{kumar2009attribute, lampert2009learning}.
Early versions did not use end-to-end neural networks, which soon overtook them in predictive accuracy.
Consequently, bottleneck models have historically been more popular for few-shot learning settings, where shared concepts might allow generalization to unseen contexts,
rather than the standard supervised setting we consider here.

More recently, deep neural networks with concept bottlenecks have re-emerged as targeted tools for solving particular tasks, e.g., \citet{de2018clinically} for retinal disease diagnosis,  \citet{yi2018neural} for visual question-answering, and \citet{bucher2018semantic} for content-based image retrieval.
\citet{losch2019interpretability} and \citet{chen2020concept} also explore learning concept-based models via auxiliary datasets.

\textbf{Feature engineering.}
Constructing a concept bottleneck model is similar to traditional feature engineering \citep{lewis1992feature, zheng2018feature, nixon2019feature} in that both require specifying intermediate concepts/features.
However, they differ in an important way: in the former, we learn mappings from raw input to high-level concepts, whereas in the latter, we construct low-level features that can be computed from the raw input by handwritten functions.

\textbf{Concepts as auxiliary losses or features.}
Non-bottleneck models that use human-specified concepts commonly use them in auxiliary objectives in a multi-task setup, or as auxiliary features;
examples include using object parts \citep{huang2016part,zhou2018interpretable}, parse trees \citep{zelenko2003kernel,bunescu2005shortest},
or natural language explanations \citep{murty2020expbert}.
However, these models do not support intervention on concepts.
For instance, consider a multi-task model $c \leftarrow x \rightarrow y$, with the concepts $c$ used in an auxiliary loss; simply intervening on $\hat{c}$ at test time will not affect the model's prediction of $y$.
Interventions do affect models that use $c$ as auxiliary features
by first predicting $x \to c$ and then predicting $(x, c) \to y$ (e.g., \citet{sutton2005joint}),
but we cannot intervene in isolation on a single concept because of the side channel from $x \to y$.

\textbf{Causal models.}
We emphasize that we study interventions on the value of a predicted concept within the model, not on that concept in reality.
In other words, we are interested in how changing the model's predicted concept values $\hat{c}$ would affect its final prediction $\hat{y}$, and not in whether intervening on the true concept value $c$ in reality would actually affect the true label $y$.
This is similar to the notion of causality explored by other concept-based interpretability methods, e.g., in \citet{goyal2019explaining} and \citet{o2020generative}.

More broadly, many others have studied learning models of actual causal relationships in the world \citep{pearl2000causality}.
While concept bottleneck models can represent causal relationships between $x \to c \to y$ if the set of concepts $c$ is chosen appropriately, they have the advantage of being flexible and do not require $c$ to cause $y$.
For example, imagine that arthritis grade ($y$) is highly correlated with swelling ($c$). In this case, $c$ does not cause $y$ (hypothetically, if one could directly induce swelling in the patient, it would not affect whether they had osteoarthritis). However, concept bottleneck models can still exploit the fact that $c$ is highly predictive for $y$, and
intervening on the model by replacing the predicted concept value $\hat{c}$ with the true value $c$ can still improve accuracy, even if $c$ does not cause $y$.

\textbf{Post-hoc concept analysis.}
Many methods have been developed to interpret models post-hoc, including recent work on using human-specified concepts to generate explanations \citep{bau2017network,kim2018interpretability,zhou2018interpretable,ghorbani2019towards}.
These techniques rely on models automatically learning those concepts despite not having explicit knowledge of them, and can be particularly useful when paired with models that attempt to learn more interpretable representations \citep{bengio2013representation,chen2016infogan,higgins2017beta,melis2018towards}.
However, post-hoc methods can fail when the models do not learn these concepts, and also do not admit straightforward interventions on concepts.
In this work, we instead directly guide models to learn these concepts at training time.

\section{Setup}\label{sec:setup}

Consider predicting a target $y \in \R$ from input $x \in \R^d$;
for simplicity, we present regression first and discuss classification later.
We observe training points $\{(x^{(i)}, y^{(i)}, c^{(i)})\}_{i=1}^n$,
where $c \in \R^k$ is a vector of $k$ concepts.
We consider bottleneck models of the form $f(g(x))$,
where $g:\R^d \to \R^k$ maps an input $x$ into the concept space (``bone spurs'', etc.),
and $f:\R^k \to \R$ maps concepts into a final prediction (``arthritis severity'').
We call these \emph{concept bottleneck models} because
their prediction $\hat{y} = f(g(x))$ relies on the input $x$ entirely through the bottleneck $\hat{c} = g(x)$,
which we train to align component-wise to the concepts $c$.
We define \emph{task accuracy} as how accurately $f(g(x))$ predicts $y$,
and \emph{concept accuracy} as how accurately $g(x)$ predicts $c$ (averaged over each concept).
We will refer to $g(\cdot)$ as predicting $x \to c$, and to $f(\cdot)$ as predicting $c \to y$.

In our work, we systematically study different ways of learning concept bottleneck models.
Let $L_{C_j}: \R \times \R \to \R_+$ be a loss function that measures the discrepancy between the predicted and true $j$-th concept,
and let $L_Y: \R \times \R \to \R_+$ measure the discrepancy between predicted and true targets.
We consider the following ways to learn a concept bottleneck model $(\hat{f}, \hat{g})$:
\begin{enumerate}
  \item The \emph{independent bottleneck} learns $\hat{f}$ and $\hat{g}$ independently:
  $\hat{f} = \argmin_{f} \sum_i L_Y(f(c^{(i)}); y^{(i)})$,
  and $\hat{g} = \argmin_{g} \sum_{i,j} L_{C_j}(g_j(x^{(i)}); c_j^{(i)})$.
  While $\hat{f}$ is trained using the true $c$, at test time it still takes $\hat{g}(x)$ as input.

  \item The \emph{sequential bottleneck} first learns $\hat{g}$ in the same way as above.
  It then uses the concept predictions $\hat{g}(x)$ to learn
  $\hat{f} = \argmin_{f} \sum_i L_Y(f(\hat{g}(x^{(i)})); y^{(i)})$.

  \item The \emph{joint bottleneck} minimizes the weighted sum
  $\hat{f}, \hat{g} = \argmin_{f,g} \sum_i \bigl[ L_Y(f(g(x^{(i)})); y^{(i)}) + \sum_j \lambda L_{C_j}(g(x^{(i)}); c^{(i)}) \bigr]$
  for some $\lambda > 0$.
\end{enumerate}

As a control, we also study the \emph{standard model}, which ignores concepts and directly minimizes $\hat{f}, \hat{g} = \argmin_{f,g} \sum_i L_Y(f(g(x^{(i)})); y^{(i)})$.

The hyperparameter $\lambda$ in the joint bottleneck controls the tradeoff between concept vs. task loss.
The standard model is equivalent to taking $\lambda \to 0$, while the sequential bottleneck can be viewed as taking $\lambda \to \infty$.
Compared to independent bottlenecks,
sequential bottlenecks allow the $c \to y$ part of the model to adapt to how well it can predict $x \to c$;
and joint bottlenecks further allow the model's version of the concepts to be refined to improve predictive performance.

We propose a simple scheme to turn an end-to-end neural network into a concept bottleneck model:
simply resize one of its layers to have $k$ neurons to match the number of concepts $k$, then choose one of the training schemes above.

\textbf{Classification.}
In classification, $f$ and $g$ compute real-valued scores (e.g., concept logits $\hat{\ell} = \hat{g}(x) \in \R^k$)
that we then turn into a probabilistic prediction
(e.g., $P(\hat{c}_j = 1) = \sigma(\hat{\ell}_j)$ for logistic regression).
This does not change the independent bottleneck,
since $f$ is directly trained on the binary-valued $c$.
For the sequential and joint bottlenecks, we connect $f$ to the logits $\hat{\ell}$,
i.e., we compute $P(\hat{c}_j = 1) = \sigma(\hat{g}_j(x))$ and $P(\hat{y} = 1) = \sigma(\hat{f}(\hat{g}(x)))$.

\begin{table}[!t]
\caption{Task errors with $\pm$2SD over random seeds.
  Overall, independent, sequential, and joint concept bottleneck models are comparable to standard end-to-end models on task error. Removing the bottleneck from the standard model (``no bottleneck'') does not significantly affect task error. Multi-task learning on the standard models further improves task error, but does not allow for interventions.
  }
\label{tab:target_acc}
\begin{center}
\begin{small}
\begin{sc}
\begin{tabular}{lcc}
\toprule
Model & $y$ RMSE (OAI) & $y$ Error (CUB)\\
\midrule
Independent & 0.435$\pm$ 0.024& 0.240$\pm$0.012  \\
Sequential  & 0.418$\pm$ 0.004& 0.243$\pm$0.006  \\
Joint       & 0.418$\pm$ 0.004& 0.199$\pm$0.006  \\
\midrule
Standard    & 0.441$\pm$ 0.006& 0.175$\pm$0.008 \\
\ \ no bottleneck   & 0.443 $\pm$ 0.008 & 0.173$\pm$0.003 \\
Multitask   & 0.425$\pm$ 0.010& 0.162$\pm$0.002 \\
\bottomrule
\end{tabular}
\end{sc}
\end{small}
\end{center}
\vspace{-5mm}
\end{table}

\section{Benchmarking bottleneck model accuracy}\label{sec:experiments}

We start by showing that concept bottleneck models achieve both competitive task accuracy and high concept accuracy.
While this is necessary for bottleneck models to be viable in practice,
their strength is that we can interpret and intervene on them;
we explore those aspects in Sections~\ref{sec:posthoc} and \ref{sec:tti}.

\subsection{Applications}\label{sec:experiments_applications}
We consider an x-ray grading and a bird identification task.
Their corresponding datasets are annotated with high-level concepts that practitioners (radiologists/birders) use to reason about their decisions.
(Dataset details in \refapp{supp_dataset}.)

\textbf{X-ray grading (OAI).}
We use knee x-rays from the Osteoarthritis Initiative (OAI) \citep{nevitt2006osteoarthritis},
which compiles radiological and clinical data on patients at risk of knee osteoarthritis (\reffig{intro}-Top; $n=36,369$ data points).
Given an x-ray, the task is to predict the Kellgren-Lawrence grade (KLG),
a 4-level ordinal variable assessed by radiologists that measures the severity of osteoarthritis,
with higher scores denoting more severe disease.\footnote{%
  Due to technicalities in the data collection protocol, we use a modified version of KLG
  where the first two grades are combined.
}
As concepts, we use $k=10$ ordinal variables describing joint space narrowing, bone spurs, calcification, etc.;
these clinical concepts are also assessed by radiologists and
used directly in the assessment of KLG \citep{kellgren1957radiological}.

\textbf{Bird identification (CUB).}
We use the Caltech-UCSD Birds-200-2011 (CUB) dataset \citep{wah2011cub}, which comprises $n = 11,788$ bird photographs (\reffig{intro}-Bot). The task is to classify the correct bird species out of 200 possible options. As concepts, we use $k = 112$ binary bird attributes representing wing color, beak shape, etc.
Because the provided concepts are noisy (see \refapp{supp_dataset}), we denoise them by majority voting, e.g., if more than $50\%$ of crows have black wings in the data, then we set all crows to have black wings.
In other words, we use class-level concepts and assume that all birds of the same species in the training data share the same concept annotations.
In contrast, the OAI dataset uses instance-level concepts: examples with the same $y$ can have different concept annotations $c$.

\textbf{Models.}
For each task, we construct concept bottleneck models by
adopting model architectures and hyperparameters from previous high-performing approaches;
see \refapp{supp_experiments} for experimental details.
For the joint bottleneck model, we search over the task-concept tradeoff hyperparameter $\lambda$ and report results
for the model that has the highest task accuracy while maintaining high concept accuracy on the validation set
($\lambda=1$ for OAI and $\lambda=0.01$ for CUB).
We model x-ray grading as a regression problem (minimizing mean squared error) on both the KLG target $y$ and concepts $c$,
following \citet{pierson2019using};
we learn $g$, which goes from $x \to c$, by fine-tuning a pretrained ResNet-18 model \citep{he2016resnet},
and we learn $f$, which goes from $c \to y$, by training a small 3-layer multi-layer perceptron.
We model bird identification as multi-class classification for the species $y$ and binary classification for the concepts $c$.
Following \citet{cui2018large}, we learn $g$ by fine-tuning an Inception-v3 network \citep{szegedy2016rethinking}, and learn $f$ by training a single linear layer (i.e., logistic regression).

\subsection{Task and concept accuracies}

\begin{figure*}[!th]
  \begin{center}
    \includegraphics[width=\textwidth]{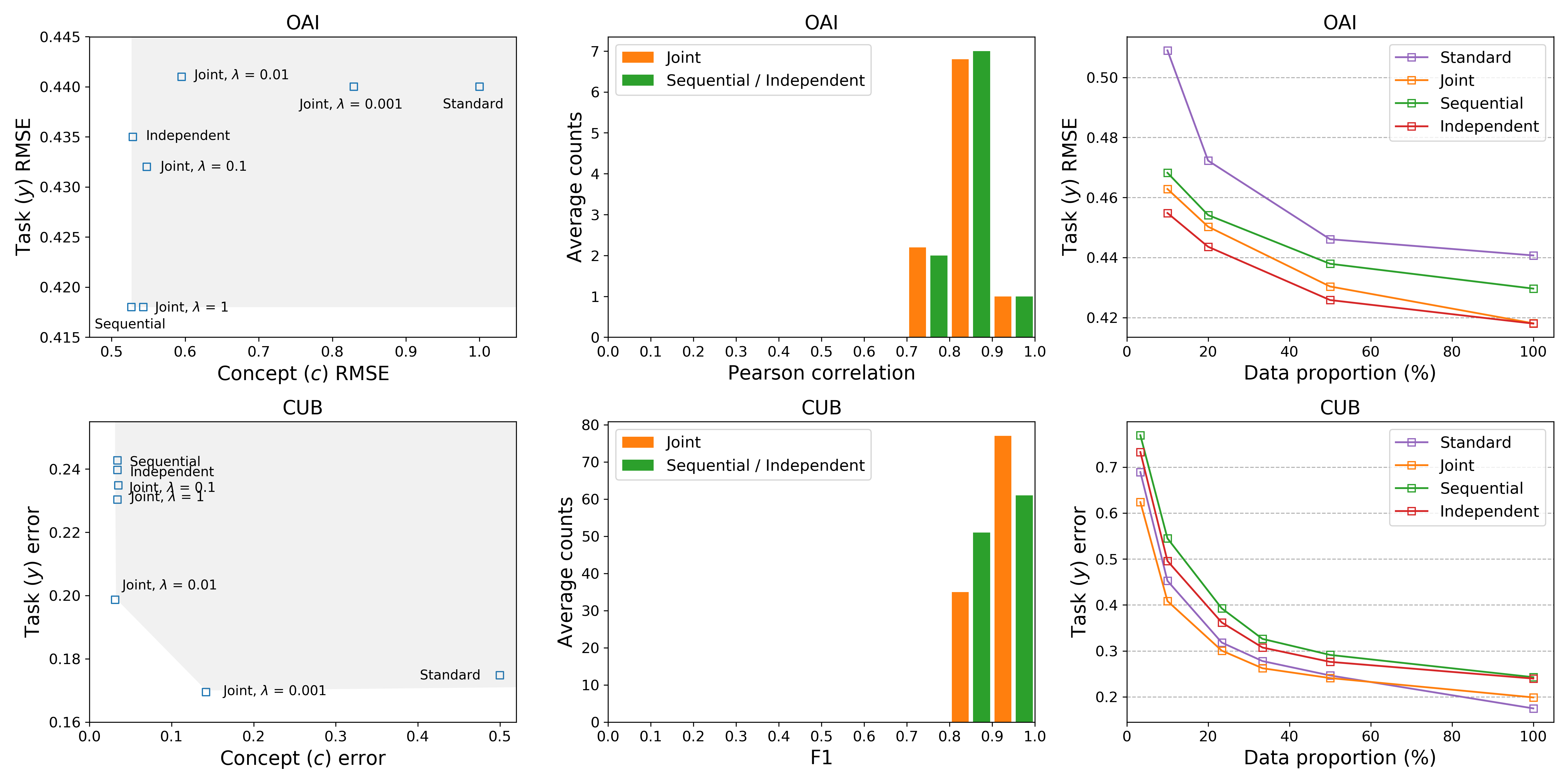}
    \vspace{-10mm}
    \caption{
      \textbf{Left}: The shaded regions show the optimal frontier between task vs. concept error.
      On OAI, we find little trade-off; models can do well on both task and concept prediction. On CUB, there is some trade-off, with standard models and joint models that prioritize task prediction (i.e., with sufficiently low $\lambda$) having lower task error.
      For standard models, we plot the concept error of the mean predictor (OAI) or random predictor (CUB).
      \textbf{Mid}: Histograms of how accurate individual concepts are, averaged over multiple random seeds.
      In our tasks, each individual concept can be accurately predicted by bottleneck models.
      \textbf{Right}: Data efficiency curves. Especially on OAI, bottleneck models can achieve the same task accuracy as standard models with many fewer training points.
    }
    \label{fig:experiments_main}
  \end{center}
\end{figure*}

\begin{table}[!t]
\caption{Average concept errors.
Bottleneck models have lower error than linear probes on standard and SENN models.}
\label{tab:concept_acc}
\begin{center}
\begin{small}
\begin{sc}
\begin{tabular}{lcccc}
\toprule
Model & $c$ RMSE (OAI) & $c$ Error (CUB) \\
\midrule
Independent & 0.529$\pm$0.004 & 0.034$\pm$0.002\\
Sequential  & 0.527$\pm$0.004 & 0.034$\pm$0.002\\
Joint       & 0.543$\pm$0.014 & 0.031$\pm$0.000\\
\midrule
Standard [probe] & 0.680$\pm$0.038 & 0.093$\pm$0.004\\
SENN [probe]     & 0.676$\pm$0.026 & - \\
\bottomrule
\end{tabular}
\end{sc}
\end{small}
\end{center}
\vspace{-5mm}
\end{table}

\reftab{target_acc} shows that concept bottleneck models achieve comparable task accuracy to standard black-box models on both tasks,
despite the bottleneck constraint (all numbers reported are on a held-out test set).
On OAI, joint and sequential bottlenecks are actually better in root mean square error (RMSE) than the standard model,\footnote{%
  To contextualize RMSE, our modified KLG ranges from 0-3, and average Pearson correlations between each predicted and true concept are ${\geq}0.87$ for all bottleneck models.
}
and on CUB, sequential and independent bottlenecks are worse in 0-1 error than the standard model, though the joint model (which is allowed to modify the concepts to improve task performance) closes most of the gap.

At the same time, the bottleneck models are able to accurately predict each concept well (\reffig{experiments_main}),
and they achieve low average error across all concepts (\reftab{concept_acc}).
As discussed in \refsec{intro}, low concept error suggests that the model's concepts are aligned with the true concepts, which in turn suggests that we might intervene effectively on them; we will explore this in \refsec{tti}.

Overall, we do not observe a tradeoff between high task accuracy and high concept accuracy:
pulling the bottleneck layer towards the concepts $c$ does not substantially affect the model's ability to predict $y$ in our tasks, even when the bottleneck is trained jointly.
We illustrate this in \reffig{experiments_main}-Left, which plots the task vs. concept errors of each model.

\textbf{Additional baselines.}
We ran two further baselines to determine if the bottleneck architecture impacted model performance.
First, standard models in the literature generally do not use architectures that have a bottleneck layer with exactly $k$ units,
so we trained a variant of the standard model without that bottleneck layer
(directly using a ResNet-18 or Inception-v3 model to predict $x \to y$);
this performed similarly to the standard bottleneck model (``Standard, no bottleneck'' in \reftab{target_acc}).
Second, we tested a typical multi-task setup using an auxiliary loss to encourage the activations of the last layer to be predictive of the concepts $c$,
hyperparameter searching across different weightings of this auxiliary loss.
These models also performed comparably (``Multitask'' in \reftab{target_acc}), but since they do not
support concept interventions,
we focus on comparing standard vs. concept bottleneck models in the rest of the paper.

\textbf{Data efficiency.}
Another way to benchmark different models is by measuring data efficiency, i.e., how many training points they need for a desired level of accuracy.
To study this, we subsampled the training and validation data and retrained each model (details in \refapp{supp_data_efficiency}).
Concept bottleneck models are particularly effective on OAI: the sequential bottleneck model with $\approx 25\%$ of the full dataset performs similarly to the standard model.
On CUB, the joint bottleneck and standard models are more accurate throughout, with the joint model slightly more accurate in lower data regimes (\reffig{experiments_main}-Right).

\section{Benchmarking post-hoc concept analysis}\label{sec:posthoc}
Concept bottleneck models are trained to have a bottleneck layer that aligns component-wise with
the human-specified concepts $c$.
For any test input $x$, we can read out predicted concepts directly from the bottleneck layer,
as well as intervene on concepts by manipulating the predicted concepts $\hat{c}$
and inspecting how the final prediction $\hat{y}$ changes.
This enables explanations like ``if the model did not think the joint
space was too narrow for this patient, then it would not have
predicted severe arthritis''.
An alternative approach to interpreting models in terms of concepts is
post-hoc analysis: take an existing model trained to directly predict $x \to y$ without any concepts,
and use a probe to recover the known concepts from the model's activations.
For example, \citet{bau2017network} measure the correlation of individual neurons with concepts,
while \citet{kim2018interpretability} use a linear probe to predict concepts with linear combinations of neurons.

A necessary condition for post-hoc interpretation is high concept accuracy.
In this section, we therefore evaluate how accurately linear probes can predict concepts post-hoc.
We emphasize that this is a necessary but not sufficient condition for accurate post-hoc interpretations, as post-hoc analysis does not enable interventions on concepts:
even if we find a linear combination of neurons that predicts a concept well,
it is unclear how to modify the model's activations to change what it thinks of that concept alone.
Without this ability to intervene, interpretations in terms of concepts are suggestive but fraught:
even if we can say that ``the model thinks the joint space is narrow'',
it is hard to test if that actually affects its final prediction.
This is an important limitation of post-hoc interpretation, though we will set it aside for this section.

Following \citet{kim2018interpretability}, we trained a linear probe to predict each concept
from the layers of the standard model (see \refapp{supp_experiments}).
We found that these linear probes have lower concept accuracy compared to simply reading concepts out from a bottleneck model
(\reftab{concept_acc}).
On OAI, the best-performing linear probe achieved an average concept RMSE of $0.68$, vs.
$0.53$ in the bottleneck models; average Pearson correlation dropped to $0.72$ from $0.84$.
On CUB, the linear probe achieved an average concept error of $0.09$ instead of $0.03$;
average F1 score dropped to $0.77$ from $0.92$.

We also tested if we could predict concepts post-hoc from models
designed to learn an interpretable mapping from $x \to y$.
Specifically, we evaluated self-explaining neural networks (SENN) \citep{melis2018towards}.
As with standard models, SENN does not use any pre-specified concepts;
it learns an input representation encouraged to be interpretable through diversity and smoothness constraints.
However, linear probes on SENN also had lower concept accuracy on OAI
($0.68$ concept RMSE; see \refapp{supp_experiments}).\footnote{%
We were unable to run SENN on CUB because the default implementation was too memory-intensive; CUB has many more classes/concepts than the tasks SENN was originally used for.}

The comparative difficulty in predicting concepts post-hoc
suggests that if we have prior knowledge of what concepts practitioners would use,
then it helps to directly train models with these concepts
instead of hoping to recover them from a model trained without knowledge of these concepts.
See \citet{chen2020concept} for a related discussion.

\section{Test-time intervention}\label{sec:tti}

The ability to intervene on concept bottleneck models enables human users to have richer interactions with them.
For example, if a radiologist disagrees with a model's prediction,
she would not only be able to inspect the predicted concepts,
but also simulate how the model would respond to changes in those predicted concepts.
This kind of \emph{test-time intervention} can be particularly useful in high-stakes settings like medicine,
or in other settings where it is easier for users to identify the concepts $c$ (e.g., wing color) than the target $y$ (exact species of bird).

We envision that in practice, domain experts interacting with the model could intervene to ``fix'' potentially incorrect concept values predicted by the model.
To study this setting, we use an oracle that can query the true value of any concept for a test input.
\reffig{tti_qual} shows examples of interventions that lead to the model making a correct prediction.

\begin{figure*}[!th]
\begin{center}
\includegraphics[width=\textwidth]{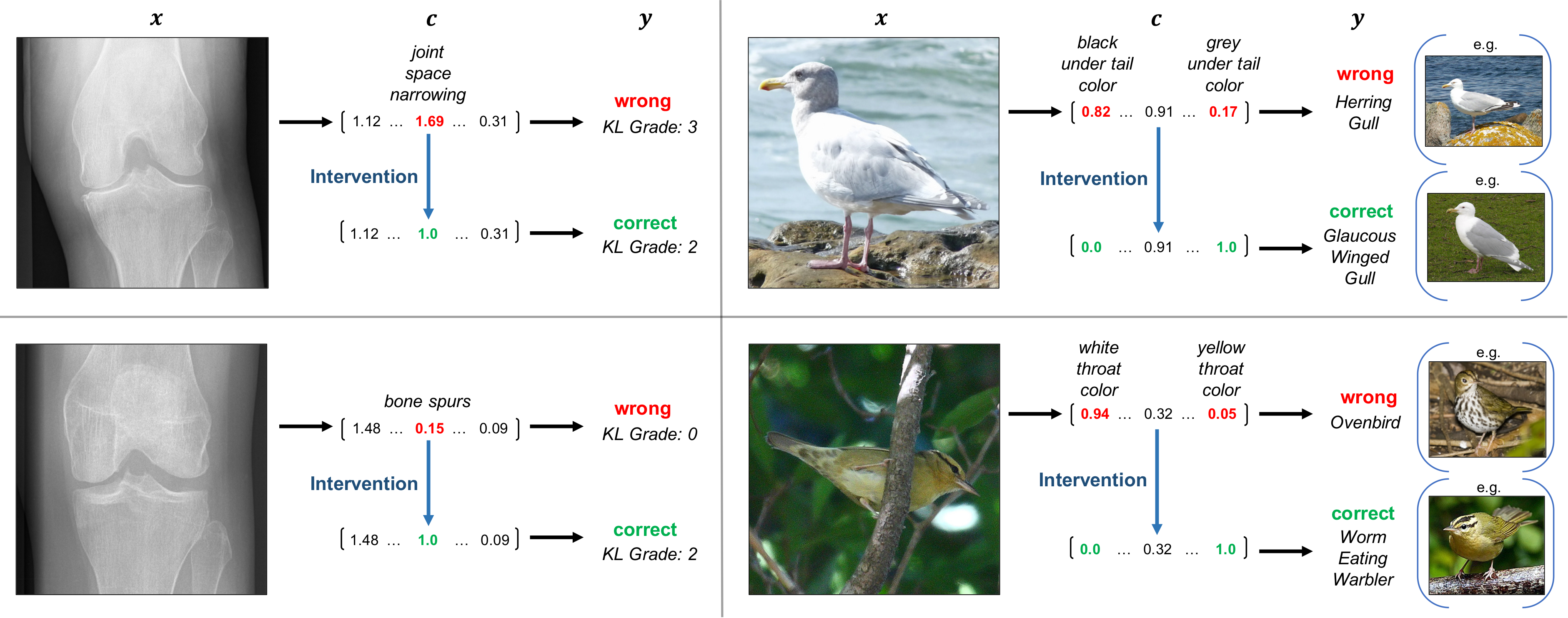}
\vspace{-12mm}
\caption{Successful examples of test-time intervention, where intervening on a single concept corrects the model prediction. Here, we show examples from independent bottleneck models. \textbf{Right}: For CUB, we intervene on concept groups instead of individual binary concepts. The sample birds on the right illustrate how the intervened concept distinguishes between the original and new predictions.}
\label{fig:tti_qual}
\end{center}
\end{figure*}

\begin{figure*}[!th]
\begin{center}
\includegraphics[width=\textwidth]{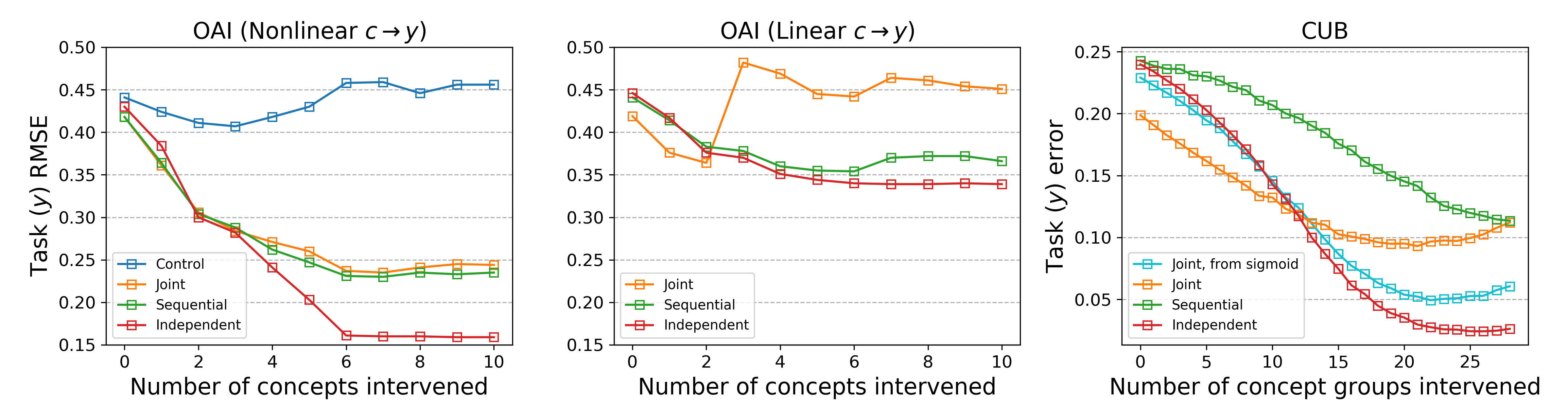}
\vspace{-10mm}
\caption{
  Test-time intervention results.
  \textbf{Left}: Intervention substantially improves task accuracy, except for the control model, which is a joint model that heavily prioritizes label accuracy over concept accuracy.
  \textbf{Mid}: Replacing $c \to y$ with a linear model degrades effectiveness.
  \textbf{Right}: Intervention improves task accuracy except for the joint model. Connecting $c \to y$ to probabilities rescues intervention but degrades normal accuracy.
  }
\label{fig:tti_quant}
\end{center}
\end{figure*}

\subsection{Intervening on OAI}
Recall that concept bottleneck models first predict concept values $\hat{c}$ from the input $x$, and then use those predicted concept values to predict the target $\hat{y}$. On OAI, we define intervening on the $j$-th concept as replacing $\hat{c}_j$ with its true value $c_j$, as provided by the oracle, and then updating the prediction $\hat{y}$ after this replacement. Similarly, we can intervene on multiple concepts by simultaneously replacing all of their corresponding predicted concept values, and then updating the prediction. (Intervening on CUB, which we describe in the next subsection, is similar but slightly more complicated because of its regression setting.)

Concretely, we iteratively select concepts on which to intervene on, using an  input-independent ordering over concepts computed from the held-out validation set. This means that we always intervene on the same concept $c_{i_1}$ first, followed by intervening on both $c_{i_1}$ and $c_{i_2}$, and so on (see \refapp{supp_experiments} for more detail).

We found that test-time intervention in this manner significantly improved task accuracy on OAI:
e.g., querying for just 2 concepts reduces task RMSE from ${>}0.4$ to ${\approx}0.3$ (\reffig{tti_quant}-Left).
These results hint that a single radiologist collaborating with bottleneck models
might be able to outperform either the radiologist or model alone,
since the concept values used for intervention mostly come from a single radiologist instead of a consensus reading (see \refapp{supp_dataset}),
and neural networks similar to ours are comparable with individual radiologist performance in terms of agreement with the consensus grade \citep{tiulpin2018automatic, pierson2019using}.
However, definitively showing this would require more careful human studies.

Furthermore, we found a trade-off between intervenability and task accuracy:
the independent bottleneck achieved better test error when all $k=10$ concepts are replaced than the sequential or joint bottlenecks, but performed slightly worse without any intervention (\reffig{tti_quant}-Left).
This behavior is consistent with how these different bottleneck models are trained.
Recall that in the independent bottleneck, $c \to y$ is trained using the true $c$, which is what we replace the predicted concepts $\hat{c}$ with.
In contrast, in the sequential and joint models, $c \to y$ is trained using the predicted $\hat{c}$,
which in general will have a different distribution from the true $c$.
Without any interventions, we might therefore expect the independent bottleneck to perform worse, as at test time, it receives the distribution over the predicted $\hat{c}$ instead of the distribution of the true $c$ that it was trained with. However, when all concepts are replaced, the reverse is true.

To better understand what influences intervention effectiveness,
we ran two ablations.
First, we found that intervention can fail in joint models when
we prioritize fitting $y$ over $c$ too much (i.e., when $\lambda$ is too small).
Specifically, the joint model with $\lambda = 0.01$ learned a concept representation that was not as well-aligned with the true concepts, and replacing $\hat{c}$ with the true $c$ at test time slightly \emph{increased} test error (``control'' model in \reffig{tti_quant}-Left).
Second, we changed the $c \to y$ model
from the 3-layer multi-layer perceptron used throughout the paper to a single linear layer.
Surprisingly, test-time intervention was less effective here compared to the non-linear counterparts (\reffig{tti_quant}-Mid),
even though task and concept accuracies were similar before intervention (concept RMSEs of the sequential and independent models are not even affected by the change in $c \to y$).
It is unclear to us why a linear $c \to y$ model should be worse at handling interventions, and this observation warrants further investigation in future work.

Altogether, these results suggest that task and concept accuracies alone are insufficient for determining how effective test-time intervention will be on a model.
Different inductive biases in different models control how effectively they can handle distribution shifts from $\hat{c} \to y$ (pre-intervention) to $c \to y$ (post-intervention).
Even without this distribution shift, as in the case of the linear vs. non-linear independent bottlenecks, the expressivity of $c \to y$ has a large effect on intervention effectiveness.
Moreover, it is possible that the average concept accuracy masks differences in individual concept accuracies that influence these results.

\subsection{Intervening on CUB}

Intervention on CUB is complicated by the fact that it is classification instead of regression.
Recall from \refsec{setup} that for sequential and joint bottleneck classifiers,
the final predictor $f$ takes in the predicted concept logits $\hat{\ell} = \hat{g}(x)$ instead of the predicted binary concepts $\hat{c}$.
To intervene on a concept $\hat{c}_j$, we therefore cannot directly copy over the true $c_j$.
Instead, we need to alter the logits $\hat{\ell}_j$ such that $P(\hat{c}_j = 1) = \sigma(\hat{\ell}_j)$ is close to the true $c_j$.
Concretely, we intervene on $\hat{c}_j$ by setting $\hat{\ell}_j$ to the 5th (if $c_j = 0$) or 95th (if $c_j = 1$) percentile of $\hat{\ell}_j$ over the training distribution.

Another difference is that for CUB, we group related concepts and intervene on them together. This is because many of the concepts encode the same underlying property, e.g.,
$c_1 = 1$ if the wing is red, $c_2 = 1$ if the wing is black, etc.
We assume that the human (oracle) returns the true wing color in a single query,
instead of only answering yes/no questions about the wing color; see \reffig{tti_quant}-Right.

An important caveat is that we use denoised class-level concepts in the CUB dataset (\refsec{experiments_applications}).
To avoid unrealistic scenarios where a bird part is not visible in the image but we still `intervene' on it, we only replace a concept value with the true concept value if that concept is actually visible in the image (visibility information is included in the dataset).
The results here are nonetheless still optimistic, because they assume that human experts do not make mistakes in identifying concepts and that birds of the same species always share the same concept values.

Test-time intervention substantially improved accuracy on CUB bottleneck models (\reffig{tti_quant}-Right),
though it took intervention on several concept groups to see a large gain.
For simplicity, we queried concept groups in random order,
which means that many queries were probably irrelevant for any given test example.\footnote{This differs from OAI, which used a fixed ordering computed on a held-out validation set. For CUB, it is common to retrain the model on the training + validation set after the hyperparameter search, which makes it difficult to subsequently use the held-out validation set. See \refapp{supp_experiments} for more details.}

As with OAI, test-time intervention was more effective on independent bottleneck models than on the sequential and joint models (\reffig{tti_quant}-Right).
We hypothesize that this is partially due to the ad-hoc fashion in which we set logits to the 5th or 95th percentiles for the latter models.
To study this, we trained a joint bottleneck with the same task-concept tradeoff $\lambda$ but with the final predictor $f$ connected to the probabilities $P(\hat{c_j} = 1) = \sigma(\hat{\ell}_j)$ instead of the logits $\hat{\ell}_j$.
We show the performance of this model under test-time intervention as the ``Joint, from sigmoid'' curve in \reffig{tti_quant}-Right.
It had a higher task error of $0.224$ vs. $0.199$ with the normal joint model;
we suspect that the squashing from the sigmoid makes optimization harder.
However, test-time intervention worked better, and it is more straightforward as we can directly edit $\hat{c}$ instead of having to arbitrarily choose a percentile for the logits $\hat{\ell}$.
This raises the question of how to effectively intervene in the classification setting while
maintaining the computational advantages of avoiding the sigmoid in the $c \to y$ connection.

\section{Robustness to background shifts}\label{sec:robustness}
Finally, we investigate if concept bottleneck models
can be more robust than standard models to spurious correlations that hold in the training distribution but not the test distribution.
For example, in the bird recognition task, a model might spuriously associate the image background with the label and therefore make more errors if the relationship between the background and label changes.
However, if the relationship between the concepts and the label remains invariant, then one might hope that concept bottleneck models might still perform well under this change.

\begin{figure}[!t]
\begin{center}
\includegraphics[width=\textwidth]{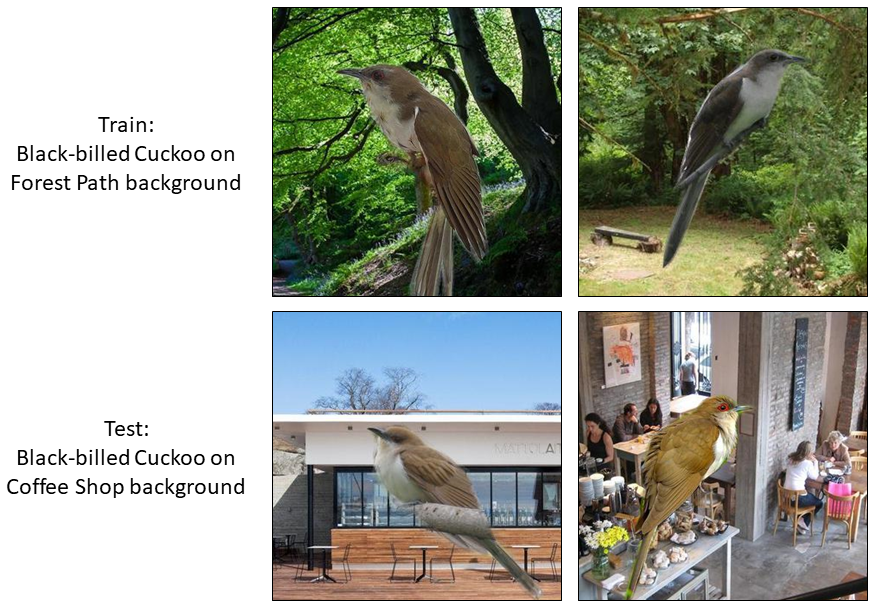}
\vspace{-5mm}
\caption{In the TravelingBirds dataset, we change the image backgrounds associated with each class from train to test time
(illustrated above for a single class).
}
\vspace{-5mm}
\label{fig:adversarial_data}
\end{center}
\end{figure}

\begin{table}[!t]
\caption{Task and concept error with background shifts.
Bottleneck models have substantially lower task error than the standard model.}
\label{tab:robust_results}
\begin{center}
\begin{small}
\begin{sc}
\begin{tabular}{lcc}
\toprule
Model & $y$ Error & $c$ Error\\
\midrule
Standard    & 0.627$\pm$0.013 & - \\
Joint       & 0.482$\pm$0.018 & 0.069$\pm$0.002 \\
Sequential  & 0.496$\pm$0.009 & 0.072$\pm$0.002 \\
Independent & 0.482$\pm$0.008 & 0.072$\pm$0.002 \\
\bottomrule
\end{tabular}
\end{sc}
\end{small}
\end{center}
\vspace{-5mm}
\end{table}

To test this, we constructed the TravelingBirds dataset, a variant of the CUB dataset where the target $y$ is spuriously correlated with image background in the training set.
Specifically, we cropped each bird out of its original background (using segmentation masks from the original dataset) and onto a new background from the Places dataset \citep{zhou2017places},
with each bird class (species) assigned to a unique and randomly-selected category of places. At test time, we shuffle this mapping, so each class is associated with a different category of places.
For example, at training time, all robins might be pictured against the sky, but at test time they might all be on grassy plains (\reffig{adversarial_data}).\footnote{TravelingBirds is constructed in a similar manner to the Waterbirds dataset introduced in \citet{sagawa2020group}. In Waterbirds, which was designed to benchmark methods for handling group shifts, the classes are collapsed into just two classes (waterbirds and landbirds), and only land or water backgrounds are used. In TravelingBirds, which is a more adversarial setting, we retain the multi-class labels, use a larger set of backgrounds, and change the test distribution so that the relationship between the backgrounds and labels are completely altered.}

The results on TravelingBirds are shown in \reftab{robust_results}.
Concept bottleneck models do better than standard models: they rely less on background features, since each concept is shared among multiple bird classes and thus appears in training data points that span multiple background types, reducing the correlation between the concept and the background. On the other hand, standard models that go straight from the input $x$ to the label $y$ leverage the spurious correlation between background and label, and consequently fail on the shifted test set.

This toy experiment shows that concept bottleneck models can be more robust to spurious correlations when the target $y$ is more correlated with training data artifacts compared to the concepts $c$.
We emphasize that whether bottleneck models are more robust depends on the choice of the set of concepts $c$ and the shifts considered; a priori, we do not expect that an arbitrary set of concepts $c$ will lead to a more robust model.

\section{Discussion}\label{sec:discussion}
Concept bottleneck models can compete on task accuracy
while supporting intervention and interpretation,
allowing practitioners to reason about these models in terms of high-level concepts they are familiar with,
and enabling more effective human-model collaboration through test-time intervention.
We believe that these models can be promising in settings like medicine,
where the high stakes incentivize human experts to collaborate with models at test time,
and where the tasks are often normatively defined with respect to a set of standard concepts
(e.g., ``osteoarthritis is marked by the presence of bone spurs'').
A flurry of recent papers have used similar human concepts for post-hoc interpretation of medical and other scientific ML models, e.g., \citet{graziani2018regression} for breast cancer histopathology; \citet{clough2019global} for cardiac MRIs; and \citet{sprague2019interpretable} for meteorology (storm prediction).
We expect that concept bottleneck models can be applied directly to similar settings.

A drawback of concept bottleneck models is that they require annotated concepts at training time.
However, if the set of concepts are good enough,
then fewer training examples might be required to achieve a desired accuracy level (as in OAI).
This allows model developers to trade off the cost of acquiring more detailed annotations against
the cost of acquiring new training examples, which can be helpful when new training examples are expensive to acquire,
e.g., in medical settings where adding training points might entail invasive/expensive procedures on patients,
but the incremental cost in asking a doctor to add annotations to data points that they already need to look at might be lower.

Below, we discuss several directions for future work.

\textbf{Learning concepts.} In tasks that are not normatively defined, we can learn the right concepts by interactively querying humans.
For example, \citet{cheng2015flock} asked crowdworkers to generate concepts to differentiate between adaptively-chosen pairs of examples,
and used those concepts to train models to recognize the artist of a painting, tell honest from deceptive reviews, and identify popular jokes.
Similar methods can also be used to refine existing concepts and make them more discriminative \citep{duan2012discovering}.

\textbf{Side channel from $x \to y$.}
We can also account for having an incomplete set of concepts by adding a direct side channel from $x \to y$ to a bottleneck model.
This is equivalent to using the concepts as auxiliary features,
and as discussed in \refsec{related},
has the drawback that we cannot cleanly intervene on a single concept,
since the $x \to y$ connection might also be implicitly reasoning about that concept.
Devising approaches to mitigate this issue would allow concept models to have high task accuracy even with an incomplete set of concepts; for example, one might consider carefully regularizing the $x \to y$ connection or using some sort of adversarial loss to prevent it from using existing concepts.

\textbf{Theoretically analyzing concept bottlenecks.}
In OAI, we found that concept bottleneck models outperform standard models on task accuracy. Better understanding when and why this happens can inform how we collect concepts or design the architecture of bottleneck models.
As an example of what this could entail, we sketch an analysis of a simple well-specified linear regression setting,
where we assume that the input $x \in \R^d$ is normally distributed, and that the concepts $c \in \R^k$ and the target $y \in \R$ are noisy linear transformations of $x$ and $c$ respectively.
Our analysis suggests that in this setting, concept bottleneck models can be particularly effective when the number of concepts $k$ is much smaller than the input dimension $d$ and when the concepts have relatively low noise compared to the target.

Concretely, we compared an independent bottleneck model (two linear regression problems for $x \to c$ and $c \to y$) to a standard model (a single linear regression problem) by deriving the ratio of their excess mean-squared-errors as the number of training points $n$ goes to infinity:
\begin{align*}
    \frac{\text{Excess error for indp bottleneck model}}{\text{Excess error for standard model}} \leq \frac{\frac{k}{d} \sigma_Y^2 + \sigma_C^2}{\sigma_Y^2 + \sigma_C^2},
\end{align*}
where $\sigma_C^2$ and $\sigma_Y^2$ are the variances of the noise in the concepts $c$ and target $y$, respectively. See \refapp{supp_proofs} for a formal statement and proof.
The asymptotic relative excess error is small when $\frac{k}{d}$ is small and $\sigma_Y^2 \gg \sigma_C^2$, i.e., when the number of concepts is much smaller than the input dimension and the concepts are less noisy than the target.

\textbf{Intervention effectiveness.}
Our exploration of the design space of concept bottleneck models
showed that the training method (independent, sequential, joint) and choice of architecture influence not just the task and concept accuracies, but also how effective interventions are.
This poses several open questions, for example:
What factors drive the effectiveness of test-time interventions?
Could adaptive strategies that query for the concepts that maximize expected information gain on a particular test example make interventions more effective?
Finally, how might we have models learn from interventions to avoid making similar mistakes in the future?

\newpage
\section*{Reproducibility}
The code for replicating our experiments is available on GitHub at \url{https://github.com/yewsiang/ConceptBottleneck}. An executable version of the CUB experiments in this paper is on CodaLab at \url{https://worksheets.codalab.org/worksheets/0x362911581fcd4e048ddfd84f47203fd2}. The TravelingBirds dataset can also be downloaded at that link. While we are unable to release the OAI dataset publicly, an application to access the data can be made at \url{https://nda.nih.gov/oai/}.

\clearpage
\section*{Acknowledgements}
We are grateful to Jesse Mu, Justin Cheng, Kensen Shi, Michael Bernstein, Rui Shu, Sendhil Mullainathan, Shyamal Buch, Yair Carmon, Ziad Obermeyer, and our anonymous reviewers for helpful advice. PWK was supported by a Facebook PhD Fellowship. YST was supposed by an IMDA Singapore Digital Scholarship. SM was supported by an NSF Graduate Fellowship. EP was supported by a Hertz Fellowship. Other funding came from the PECASE Award. Toyota Research Institute (``TRI'') provided funds to assist the authors with their research but this article solely reflects the opinions and conclusions of its authors and not TRI or any other Toyota entity.

\bibliography{refdb/all}
\bibliographystyle{icml2020}

\clearpage
\appendix
\onecolumn
\section{Datasets}\label{sec:supp_dataset}

\subsection{Osteoarthritis Initiative (OAI)}

\textbf{Description and statistics.}
The source of the knee x-ray dataset is the Osteoarthritis Initiative\footnote{ \url{https://nda.nih.gov/oai/ }}, which compiles radiological and clinical data on patients who have or are at high risk of developing knee osteoarthritis. We follow the dataset processing procedure used by \citet{pierson2019using} in their previous analysis. They analyzed data from the baseline visit and four follow-up timepoints (12-, 24-, 36-, and 48-month follow-ups). Two types of data from this dataset were used in our analysis: the knee x-rays, which served as the input to the neural network, and the clinical concepts associated with osteoarthritis, which were annotated by radiologists for each knee x-ray.

After filtering for observations which contain basic demographic and clinical data, the dataset contains 4,172 patients and 36,369 observations, where an observation is one knee for one patient at one timepoint. We randomly divided patients into training, validation, and test sets, with no overlap in the patient groups. Specifically, we have 21,340 observations from 2,456 people in the training set; 3,709 observations from 421 people in the validation set; and 11,320 observations from 1,295 people in the test set.

\textbf{Image processing.}
To process the knee x-rays, each x-ray was downsampled to 512 x 512 pixels and normalized by dividing pixel values by the maximum pixel value (so all pixel values were in the range 0-1) and then z-scoring. Images were removed if they did not pass OAI x-ray image metadata quality control filters.

\textbf{Clinical concept assessment and KLG merging.}
The primary clinical image feature used in analysis is Kellgren-Lawrence grade (KLG), a 5-level categorical variable (0 to 4) which is assessed by radiologists and used as a standard measure of radiographic osteoarthritis severity, with higher scores denoting more severe disease. In addition to KLG, each knee image is also assessed for 18 other clinical concepts (features) of osteoarthritis in various knee compartments, describing joint space narrowing (JSN), osteophytes, chondrocalcinosis, subchondral sclerosis, cysts, and attrition.

The Osteoarthritis Initiative only assessed these additional 18 clinical concepts (besides joint space narrowing, which is available for all participants) for participants with KLG $\geq$ 2 (a standard threshold for radiographic osteoarthritis) in at least one knee at any time point.
Therefore, in their analysis (and in this paper), \citet{pierson2019using} set these clinical concepts to zero for other participants. This corresponds to assuming that that participants who were never assessed to have osteoarthritis, and thus were not assessed for other clinical concepts, did not display those features. This procedure also means it is impossible to use the clinical concepts to distinguish most x-rays with KLG $=$ 0 from those with KLG $=$ 1 in the dataset.
To evaluate concept bottleneck models on this dataset, we therefore merged the KLG $=$ 0 and KLG $=$ 1 classes into a single level and translated the other KLG levels downwards by 1, leading to a 4-level categorical variable (0 to 3).

\textbf{Concept processing.}
Some of the clinical concepts are very sparse, with almost all x-rays in the dataset showing an absence of the associated radiographic feature. We found that there were insufficient positive training examples to be able to accurately predict these concepts; moreover,
including these sparse concepts in the bottleneck models lowered the accuracy of KLG prediction.
We therefore filtered out the clinical concepts for which the dominant class (corresponding to an absence of the feature) represents $\geq 95\%$ of the training data.

This procedure kept 10 clinical concepts: ``osteophytes femur medial'', ``sclerosis femur medial'', ``joint space narrowing medial'', ``osteophytes tibia medial'', ``sclerosis tibia medial'', ``osteophytes femur lateral'', ``sclerosis femur lateral'', ``joint space narrowing lateral'', ``osteophytes tibia lateral'', and ``sclerosis tibia lateral''. It filtered 8 concepts: ``cysts femur medial'', ``chondrocalcinosis medial'', ``cysts tibia medial'', ``attrition tibia medial'', ``cysts femur lateral'', ``chondrocalcinosis lateral'', ``cysts tibia lateral'', ``attrition tibia lateral''.

After filtering, we z-scored the remaining clinical concepts using the training set to bring them onto the same scale.

Some of the clinical concepts, such as joint space narrowing, are annotated with fractional grades (e.g., 1.2, 1.4, 1.6 etc.) in the dataset. These partial grades represent temporal progression and cannot be deduced by looking at a single timepoint, and they explicitly do not reflect fractional grades (e.g., 1.2 on one patient does not mean it is worse than 1.0 on another patient); we therefore truncate these fractional grades.

\textbf{Reader disagreements and adjudication procedures.}
KLG was read by two expert readers (i.e., radiologists) for each x-ray. Discrepancies in these readings, if they met the adjudication criteria described below, were adjudicated by a third reader: if the third reading agreed with either of the existing readings, then that reading was taken to be final, and otherwise, the three readers attended an adjudication session to form a consensus reading. If discrepancies were not adjudicated, the final reading was taken to be the one from the more senior reader. KLG readings were adjudicated when they disagreed on whether KLG was within 0-1 or 2-4, or when they
there was a difference in the direction of change of KLG between time points.

JSN was also read by two readers, with similar adjudication procedures. Discrepancies were adjudicated if the readers did not agree on the direction of change between time points.

All other clinical concepts in our dataset were read by a single reader. For more information on the adjudication procedures, please refer to the OAI documentation on Project 15.

\subsection{Caltech-UCSD Birds-200-2011 (CUB)}

\textbf{Description and statistics.}
The Caltech-UCSD Birds-200-2011 (CUB) dataset \citep{wah2011cub} comprises 11,788 photographs of birds from 200 species, with each image additionally annotated with 312 binary concepts (before processing) corresponding to bird attributes like wing color, beak shape, etc. Visibility information on each concept is also provided for each image (e.g., is the beak visible in this image?); we use this information to make our test-time intervention experiments more realistic, but not at training time. Since the original dataset only has train and test sets, we randomly split 20\% of the data from the official train set to make a validation set.

\textbf{Concept processing.}
The individual concept annotations are noisy: each annotation was provided by a single crowdworker (not a birding expert), and the concepts can be quite similar to each other, e.g., some crowdworkers might indicate that birds from some species have a red belly, while others might say that the belly is rufous (reddish-brown) instead.

To deal with this issue, we aggregate instance-level concept annotations into class-level concepts via majority voting: e.g., if more than $50\%$ of crows have black wings in the data, then we set all crows to have black wings. This makes the approximation that all birds of the same species in the training data should share the same concept annotations.
While this approximation is mostly true for this dataset, there are some exceptions due to visual occlusion, as well as sexual and age dimorphism.

After majority voting, we further filter out concepts that are too sparse, keeping only concepts (binary attributes) that are present after majority voting in at least 10 classes. After this filtering, we are left with 112 concepts.

\section{Experimental details}\label{sec:supp_experiments}

\subsection{OAI model architecture and training}
The models we use to predict KLG from knee x-rays follow the hyperparameters and model setup used by \citet{pierson2019using}, except for the learning rate and learning rate schedule, which we tune separately. Our models use a ResNet-18 \citep{he2016resnet} pretrained on ImageNet, with the last 12 convolutional layers fine-tuned on the OAI dataset.

For the bottleneck models, the ResNet-18 network extracts high-level features from the image that is used to regress to the concepts $c$ with a single fully-connected layer. Subsequently, there is a 3-layer MLP, with a dimensionality of 50 for the first two layers, that is used to regress to the final KLG $y$. The standard model is similar, except without any loss term that encourages the bottleneck layer to align with the concepts.

For fine-tuning, we use a batch size of 8, with random horizontal and vertical translations as data augmentation. Network weights are optimized with Adam, with beta parameters of 0.9 and 0.999 and an initial learning rate determined by grid search over [0.00005, 0.0005, 0.005], which decays by a factor of 2 every 10 epochs. The network is trained for 30 epochs with early stopping; model weights are set at the conclusion of training to those after the epoch with lowest RMSE for KLG on the validation set.

\subsection{CUB model architecture and training}
The main architecture for fine-grained bird classification is Inception V3, pretrained on ImageNet (except for the fully-connected layers) and then finetuned end-to-end on the CUB dataset. We follow the preprocessing practices described in \citet{cui2018large}. Each image used for training is augmented with random color jittering, random horizontal flip and random cropping with a resolution of 299. During inference, the original image is center-cropped and resized to 299.

For each model, we hyperparameter search on the validation set over a range of learning rates ([0.001, 0.01]), learning rate schedules (keeping learning rate constant or reducing learning rate by 10 times after every [10, 15, 20] epochs until it reaches 0.0001), and regularization strengths ([0.0004, 0.00004]), to find a good hyperparameter configuration. The best model is decided based on task accuracy (or concept accuracy for the $x \to c$ part of sequential models) on the validation set. Once we have found the best-performing hyperparameter configuration, we then retrain the model on both the train and validation sets until convergence, following \citet{cui2018large}.

All training is done with a batch size of 64, and SGD with momentum of 0.9 as the optimizer.
For bottleneck models, we weight each concept's contribution to the overall concept loss equally (which is in turn determined by $\lambda$ for joint bottleneck models).
However, the binary cross-entropy loss used for each individual concept prediction task is weighted by the ratio of class imbalance for that individual concept (which is about 1 : 9 on average) and normalized accordingly. This encourages the model to learn to predict positive concept labels, which are more rare, instead of mostly predicting negative labels.

\subsection{Test-time intervention}
\textbf{OAI.} For OAI, we use the held-out validation set to determine an input-independent ordering for concept intervention. Specifically, we use the concept labels in the validation set to intervene separately on each concept, replacing a single value in our original concept predictions with that ground truth concept. We obtain the intervention ordering by sorting the concepts in descending order of the improvement in KLG accuracy gained from intervening separately on each concept.

\textbf{CUB.} For CUB, the concept groups are determined by having a common prefix in the list of concept names. For example, ``has\_back\_pattern::solid'', ``has\_back\_pattern::spotted'', ``has\_back\_pattern::striped'', ``has\_back\_pattern::multi-colored'' all describe the same group that concerns back-pattern.
Since all models are retrained on both train and validation sets, as described above, we do not follow the OAI procedure of determining a fixed ordering.
Instead, we randomly select concept groups to intervene on at test time, using the class-level labels for all concepts within that group to replace the predicted logits. To avoid intervening on concepts that are not even visible in the image, we use the concept visibility information that comes with the official CUB dataset: for all concepts that are not visible in a given test image, their corrected values are set to 0 regardless of what the corresponding class-level labels may be.

\subsection{Data efficiency}\label{sec:supp_data_efficiency}
For OAI, we subsampled the training and validation data uniformly at random. For CUB, to ensure that each of the 200 classes had similar numbers of examples, we subsampled the images from each class uniformly at random.
To avoid the computational load of hyperparameter searching for each model and degree of subsampling, we adopted the hyperparameters chosen for the best-performing models on the full dataset but did early stopping on the subsampled validation datasets.

\subsection{Linear probes}\label{sec:supp_linear_probes}

\textbf{Standard (end-to-end) models.} For OAI, we separately trained linear probes on the outputs after every ResNet block and the fully-connected layers of the MLP of the standard model. The best-performing linear probe was the one trained on the output of the final ResNet block. For CUB, we ran a linear probe on the fully-connected layer of the standard model, since the $c \to y$ part of the bottleneck models are linear.

\textbf{SENN.}
To evaluate self-explaining neural networks (SENNs) \citep{melis2018towards}, we first trained a SENN model to predict KLG on the OAI dataset and then trained linear probes on the concept layer in the SENN model. We used the open-source implementation from the authors of SENN,\footnote{\url{https://github.com/dmelis/SENN}} and therefore used a classification objective for KLG prediction. To match the expressiveness of our bottleneck models,
we swapped the small CNNs of the SENN concept encoder and relevance parameterizer with our ResNet-18 models. Similarly, for the decoder network in SENN, we used a more expressive decoder comprising 2 fully-connected layers with batch normalization, followed by 5 transposed convolutional layers with upsampling. The decoder was obtained by adapting a public auto-encoder implementation,\footnote{\url{https://github.com/arnaghosh/Auto-Encoder}} changing the dimensionalities of the fully-connected and transposed convolutional layers, and increasing upsampling layers to match our input image size.
We set the number of concepts for SENN to 10, corresponding to the number of clinical features in OAI.
The learning rate was set to 0.0005 and the batch size was set to 4, which was the maximum possible given the memory constraints. With the above settings, the experiments were ran with two different seeds.

\section{Excess errors of independent vs. standard models}\label{sec:supp_proofs}

We present an analysis of the independent bottleneck model, which uses concepts at training time, versus the standard model, which does not. For simplicity, we consider a well-specified linear regression setting with normally-distributed inputs $X \in \R^d$, concepts $C \in \R^k$, and target $Y \in \R$:
\begin{align}
  X &\sim N(0, \sigma_X^2 I_d)\\
  C &= XB + \epsilon_1,\\
  Y &= Cb + \epsilon_2,
\end{align}
where $\epsilon_1 \sim N(0, \sigma_C^2 I_k)$ and $\epsilon_2 \sim N(0, \sigma_Y^2)$.
In contrast to the main text, we use capital letters for $X$, $C$, and $Y$ here to emphasize the fact that the input, concepts, and target are random variables.
In words, the input $X$ is a normally distributed with dimension $d$;
the concepts $C$ of dimension $k$ are a linear transformation of $X$ with additive Gaussian noise; and the output $Y$ is a scalar-valued linear transformation with additive Gaussian noise.
For analytical simplicity, we require $\|b\|^2=1$ and $B^\top B = I_k$.

\textbf{Independent bottleneck model.}
In this setting, the independent bottleneck model comprises two linear regression models: the first estimates the matrix $B$ that takes $X \to C$, and the second estimates the vector $b$ that takes $C \to Y$.
For ease of analysis, we assume that the linear regression models are fit using least squares on separate datasets:
the first dataset has $n_1$ training points in data matrices $\underline{X} \in \R^{n_1 \times d}$ and $\underline{C} \in \R^{n_1 \times k}$, and the second dataset has $n_2$ points in data matrices $\overline{C} \in \R^{n_2 \times k}$ and $\overline{Y} \in \R^{n_2}$.
Concretely, we estimate
\begin{align}
    \hat{B} &= \left( \underline{X}^\top \underline{X} \right)^{-1} \underline{X}^\top \underline{C} \\
    \hat{b} &= \left( \overline{C}^\top \overline{C} \right)^{-1} \overline{C}^\top \overline{Y}
\end{align}
and then compose these estimators into the final prediction $\hat{Y}_{IB} = X \hat{B} \hat{b}$.

\textbf{Standard model.}
In contrast, the standard model does not use concepts, and uses only one dataset with $n$ points in $\underline{X} \in \R^{n \times d}$ and $\underline{Y} \in \R^n$.
Concretely, we can express $Y$ directly in terms of $X$ as $Y = Xv + \epsilon$,
where $v=Bb$ and $\epsilon \sim N(0, \sigma_C^2 + \sigma_Y^2)$.
This gives the least squares estimate
\begin{align}
  \hat{v} = (\underline{X}^\top \underline{X})^{-1} \underline{X}^\top \underline{Y}
\end{align}
and the resulting prediction $\hat{Y}_{SM} = X \hat{v}$.

\textbf{Excess errors.}
We compare these two models using their asymptotic excess error as the number of training points $n_1=n_2=n$ goes to infinity, where a model's excess error is defined as how much higher its mean-squared-error is compared to the optimal estimator  $\mathbb{E}[Y|X]$.

\begin{proposition}[Relative excess error of independent bottleneck models vs. standard models in linear regression]
    Let $n_1=n_2=n$ tend to infinity. Then the ratio of excess errors of the independent bottleneck model to the standard model in the well-specified linear regression setting above is
    \begin{align}
        \lim_{n \rightarrow \infty} \frac{\mathbb{E}[ (Y - \hat{Y}_{IB})^2 ] - \mathbb{E}[ (Y - \mathbb{E}[Y|X])^2 ]}{\mathbb{E}[ (Y - \hat{Y}_{SM})^2 ] - \mathbb{E}[ (Y - \mathbb{E}[Y|X])^2 ]} \leq \frac{\frac{k}{d} \sigma_Y^2 + \sigma_C^2}{\sigma_Y^2 + \sigma_C^2}.\nonumber
    \end{align}
\end{proposition}

Note that asymptotic relative excess error is small---i.e., the independent bottleneck has lower excess error than the standard model---when $\frac{k}{d}$ is small and $\sigma_Y^2 \gg \sigma_C^2$. This corresponds to low dimensional concepts (relative to the input dimension) and concepts with low noise (relative to the noise in the output).

To prove this proposition, we first derive the expected errors of the independent bottleneck model and the standard model.

\begin{lemma}[Risk of the independent bottleneck model]
\begin{align}
\mathbb{E}[ (Y - \hat{Y}_{IB})^2 ] = \sigma_C^2 + \sigma_Y^2 + \sigma_Y^2 \frac{\sigma_X^2}{\sigma_X^2 + \sigma_C^2} \frac{k}{n_2-k-1} + \sigma_C^2 \frac{d}{n_1 - d - 1} + \sigma_Y^2 \frac{1}{\sigma_X^2 + \sigma_C^2} \frac{1}{n_2 - k - 1} \sigma_C^2 \frac{d}{n_1 - d - 1}.\nonumber
\end{align}
\end{lemma}
\begin{proof}
A direct calculation gives
\begin{align}
    \mathbb{E}[ (Y - \hat{Y}_{IB})^2 ] &= \mathbb{E}[ ((XBb + \epsilon_1 b + \epsilon_2) - X \hat{B} \hat{b})^2] \\
    &= \mathbb{E}[(\epsilon_1 b + \epsilon_2 + X(Bb - \hat{B} \hat{b}))^2] \\
    &= \mathbb{E}[ (\epsilon_1 b + \epsilon_2)^2 ] + \mathbb{E}[ X(Bb - \hat{B}\hat{b}) (Bb - \hat{B} \hat{b})^\top X^\top ] \\
    &= \sigma_C^2 + \sigma_Y^2 + \text{tr} \left( \mathbb{E}[X^\top X] \mathbb{E}[ (Bb - \hat{B}\hat{b}) (Bb - \hat{B} \hat{b})^\top ] \right) \\
    &= \sigma_C^2 + \sigma_Y^2 + \sigma_X^2 \text{tr} \left( \mathbb{E}[ (Bb - \hat{B}\hat{b}) (Bb - \hat{B} \hat{b})^\top ] \right),
\end{align}
where
\begin{align}
    (\hat{B}\hat{b} - Bb) &= \left( \underline{X}^\top \underline{X} \right)^{-1} \underline{X}^\top (\underline{X} B + \underline{\epsilon_1}) \left( \overline{C}^\top \overline{C} \right)^{-1} \overline{C}^\top (\overline{C} b + \overline{\epsilon_2}) - Bb \\
    &= \left(B + \left( \underline{X}^\top \underline{X} \right)^{-1} \underline{X}^\top \underline{\epsilon_1} \right) \left(b + \left( \overline{C}^\top \overline{C} \right)^{-1} \overline{C}^\top \overline{\epsilon_2} \right) - Bb \\
    &= B \left( \overline{C}^\top \overline{C} \right)^{-1} \overline{C}^\top \overline{\epsilon_2} + \left( \underline{X}^\top \underline{X} \right)^{-1} \underline{X}^\top \underline{\epsilon_1} b + \left( \underline{X}^\top \underline{X} \right)^{-1} \underline{X}^\top \underline{\epsilon_1} \left( \overline{C}^\top \overline{C} \right)^{-1} \overline{C}^\top \overline{\epsilon_2}.
\end{align}

We need to evaluate the expectation of this expression multiplied with itself, $\mathbb{E}[ (Bb - \hat{B}\hat{b}) (Bb - \hat{B} \hat{b})^\top ]$.
Note that the cross terms will cancel since $\underline{\epsilon_1}$ and $\overline{\epsilon_2}$ are independent of other random variables and have mean $0$, $\mathbb{E}[\underline{\epsilon_1}]=\mathbb{E}[\overline{\epsilon_2}]=0$.
This leaves three remaining direct (squared) terms, which we can evaluate separately since $\text{tr}$ and $\mathbb{E}$ are linear operators.

The first term is
\begin{align}
    & \text{tr}\left( \mathbb{E}\left[ B \left( \overline{C}^\top \overline{C} \right)^{-1} \overline{C}^\top \overline{\epsilon_2} \overline{\epsilon_2}^\top \overline{C} \left( \overline{C}^\top \overline{C} \right)^{-1} B^\top \right] \right) \\
    &= \text{tr}\left( \mathbb{E}\left[ \overline{C}^\top \left( \overline{C}^\top \overline{C} \right)^{-1} B^\top B \left( \overline{C}^\top \overline{C} \right)^{-1} \overline{C}^\top \right] \mathbb{E} \left[\overline{\epsilon_2} \overline{\epsilon_2}^\top \right] \right) \\
    &= \text{tr} \left( \mathbb{E}\left[ \overline{C}^\top \left( \overline{C}^\top \overline{C} \right)^{-1} I_k \left( \overline{C}^\top \overline{C} \right)^{-1} \overline{C}^\top\right] \sigma_Y^2 I_{n_2} \right) \\
    &= \sigma_Y^2 \text{tr} \left( \mathbb{E}\left[ \left( \overline{C}^\top \overline{C} \right)^{-1}\right]  \right).
\end{align}
The expression within the above expectation is distributed as an inverse Wishart distribution, and therefore
\begin{align}
     & \sigma_Y^2 \text{tr} \left( \mathbb{E}\left[ \left( \overline{C}^\top \overline{C} \right)^{-1}\right]  \right) \\
     &= \sigma_Y^2 \text{tr} \left( \frac{\mathbb{E}\left[C^\top C\right]^{-1}}{n_2 - k - 1} \right) \\
     &= \sigma_Y^2 \frac{1}{\sigma_X^2 + \sigma_C^2} \frac{k}{n_2-k-1},
\end{align}
where the last equality comes from $\mathbb{E}[C^\top C] = \sigma_X^2 B^\top B + \sigma_C^2 I_k = (\sigma_X^2 + \sigma_C^2) I_k$.

The second term follows a similar calculation:
\begin{align}
    &\text{tr} \left( \mathbb{E}\left[ \left( \underline{X}^\top \underline{X} \right)^{-1} \underline{X}^\top \underline{\epsilon_1} b b^\top \underline{\epsilon_1}^\top \underline{X} \left( \underline{X}^\top \underline{X} \right)^{-1} \right] \right) \\
    &= \text{tr} \left( \mathbb{E}\left[ \underline{X}^\top \left( \underline{X}^\top \underline{X} \right)^{-1}  \left( \underline{X}^\top \underline{X} \right)^{-1} \underline{X}^\top \right] \mathbb{E}\left[ \underline{\epsilon_1} b b^\top \underline{\epsilon_1}^\top   \right] \right) \\
    &= \sigma_C^2 \text{tr} \left( \mathbb{E}\left[  \left( \underline{X}^\top \underline{X} \right)^{-1} \right] \right) \\
    &= \sigma_C^2 \frac{1}{\sigma_X^2} \frac{d}{n_1 - d -1},
\end{align}
where the second equality follows because $\underline{\epsilon_1} b$ is normally distributed with mean $0$ and covariance $\sigma_C^2 I_{n_1}$.

The third term is
\begin{align}
    &\text{tr} \left( \mathbb{E}\left[ \left( \underline{X}^\top \underline{X} \right)^{-1} \underline{X}^\top \underline{\epsilon_1} \left( \overline{C}^\top \overline{C} \right)^{-1} \overline{C}^\top \overline{\epsilon_2} \overline{\epsilon_2}^\top \overline{C} \left( \overline{C}^\top \overline{C} \right)^{-1} \underline{\epsilon_1}^\top \underline{X} \left( \underline{X}^\top \underline{X} \right)^{-1} \right] \right) \\
    &= \text{tr} \left( \mathbb{E}\left[ \overline{C} \left( \overline{C}^\top \overline{C} \right)^{-1} \underline{\epsilon_1}^\top \underline{X} \left( \underline{X}^\top \underline{X} \right)^{-1} \left( \underline{X}^\top \underline{X} \right)^{-1} \underline{X}^\top \underline{\epsilon_1} \left( \overline{C}^\top \overline{C} \right)^{-1} \overline{C}^\top \overline{\epsilon_2} \overline{\epsilon_2}^\top \right] \right) \\
    &= \sigma_Y^2 \text{tr} \left( \mathbb{E}\left[  \underline{\epsilon_1}^\top \underline{X} \left( \underline{X}^\top \underline{X} \right)^{-1} \left( \underline{X}^\top \underline{X} \right)^{-1} \underline{X}^\top \underline{\epsilon_1} \left( \overline{C}^\top \overline{C} \right)^{-1}  \right] \right) \\
    &= \sigma_Y^2 \frac{1}{\sigma_X^2 + \sigma_C^2} \frac{1}{n_2 - k - 1} \text{tr} \left( \mathbb{E}\left[  \underline{\epsilon_1}^\top \underline{X} \left( \underline{X}^\top \underline{X} \right)^{-1} \left( \underline{X}^\top \underline{X} \right)^{-1} \underline{X}^\top \underline{\epsilon_1}  \right] \right) \\
    &= \sigma_Y^2 \frac{1}{\sigma_X^2 + \sigma_C^2} \frac{1}{n_2 - k - 1} \sigma_C^2 \text{tr} \left( \mathbb{E}\left[ \underline{X} \left( \underline{X}^\top \underline{X} \right)^{-1} \left( \underline{X}^\top \underline{X} \right)^{-1} \underline{X}^\top  \right] \right) \\
    &= \sigma_Y^2 \frac{1}{\sigma_X^2 + \sigma_C^2} \frac{1}{n_2 - k - 1} \sigma_C^2 \frac{1}{\sigma_X^2} \frac{d}{n_1 - d - 1}.
\end{align}

Putting the three terms together,
\begin{align}
    &\text{tr} \left( \mathbb{E}[ (\hat{B}\hat{b} - Bb) (\hat{B} \hat{b} - Bb)^\top ] \right) \\
    &= \sigma_Y^2 \frac{1}{\sigma_X^2 + \sigma_C^2} \frac{k}{n_2-k-1} + \sigma_C^2 \frac{1}{\sigma_X^2} \frac{d}{n_1 - d -1} + \sigma_Y^2 \frac{1}{\sigma_X^2 + \sigma_C^2} \frac{1}{n_2 - k - 1} \sigma_C^2 \frac{1}{\sigma_X^2} \frac{d}{n_1 - d - 1},
\end{align}
so the expected squared error is
\begin{align}
     &\mathbb{E}[ (Y - \hat{Y}_{IB})^2 ] \\
     &=\sigma_C^2 + \sigma_Y^2 + \sigma_Y^2 \frac{\sigma_X^2}{\sigma_X^2 + \sigma_C^2} \frac{k}{n_2-k-1} + \sigma_C^2 \frac{d}{n_1 - d - 1} + \sigma_Y^2 \frac{1}{\sigma_X^2 + \sigma_C^2} \frac{1}{n_2 - k - 1} \sigma_C^2 \frac{d}{n_1 - d - 1}.
\end{align}
\end{proof}

\begin{lemma}[Risk of the standard model]
\begin{align}
\mathbb{E}[ (Y - \hat{Y}_{SM})^2 ] = \sigma_C^2 + \sigma_Y^2 + \frac{d (\sigma_C^2 + \sigma_Y^2)}{n-d-1}.\nonumber
\end{align}
\end{lemma}
\begin{proof}
A direct calculation gives
\begin{align}
    \mathbb{E}[ (Y - \hat{Y}_{SM})^2 ] &= \mathbb{E}[((Xv + \epsilon) - X\hat{v})^2] \\
    &= \mathbb{E}[(\epsilon + X(v - \hat{v}))^2]\\
    &= \mathbb{E}[\epsilon^2] + \mathbb{E}[X(v-\hat{v})(v-\hat{v})^\top X^\top] \\
    &= \sigma_C^2 + \sigma_Y^2 + \text{tr} \left( \mathbb{E}[X^\top X] \mathbb{E}[(v-\hat{v})(v-\hat{v})^\top] \right) \\
    &= \sigma_C^2 + \sigma_Y^2 + \sigma_X^2 \text{tr} \left( \mathbb{E}[(v-\hat{v})(v-\hat{v})^\top] \right).
\end{align}
Since
\begin{align}
    \hat{v} - v &= (\underline{X}^\top \underline{X})^{-1} \underline{X}^\top (\underline{X} v + \underline{\epsilon}) - v \\
    &= (\underline{X}^\top \underline{X})^{-1} \underline{X}^\top \underline{\epsilon},
\end{align}
we have
\begin{align}
    \text{tr} \left( \mathbb{E}[(v-\hat{v})(v-\hat{v})^\top] \right) &= \text{tr} \left( \mathbb{E}[ (\underline{X}^\top \underline{X})^{-1} \underline{X}^\top \underline{\epsilon} \underline{\epsilon}^\top \underline{X} (\underline{X}^\top \underline{X})^{-1}] \right) \\
    &= \text{tr} \left( \mathbb{E}[\underline{X} (\underline{X}^\top \underline{X})^{-1} (\underline{X}^\top \underline{X})^{-1} \underline{X}^\top] \mathbb{E}[\underline{\epsilon} \underline{\epsilon}^\top  ] \right) \\
    &= (\sigma_C^2 + \sigma_Y^2) \text{tr} \left( \mathbb{E}[ (\underline{X}^\top \underline{X})^{-1} ] \right) \\
    &= (\sigma_C^2 + \sigma_Y^2) \frac{1}{\sigma_X^2} \frac{d}{n - d - 1}.
\end{align}

Plugging this back into the expression for $\mathbb{E}[ (Y - \hat{Y}_{SM})^2 ]$ yields
\begin{align}
    \mathbb{E}[ (Y - \hat{Y}_{SM})^2 ] &= \sigma_C^2 + \sigma_Y^2 + \frac{d (\sigma_C^2 + \sigma_Y^2)}{n-d-1}.
\end{align}

\end{proof}

\begin{proof}[Proof of Proposition 1]

Note that the optimal estimator has risk
\begin{align}
    \mathbb{E}[ (Y - \mathbb{E}[Y|X])^2 ] &= \mathbb{E}[ \epsilon^2 ] \\
    &= \sigma_C^2 + \sigma_Y^2.
\end{align}
Thus, from Lemmas 1 and 2, the ratio of excess errors is
\begin{align}
    &\frac{\mathbb{E}[ (Y - \hat{Y}_{IB})^2 ] - \mathbb{E}[ (Y - \mathbb{E}[Y|X])^2 ]}{\mathbb{E}[ (Y - \hat{Y}_{SM})^2 ] - \mathbb{E}[ (Y - \mathbb{E}[Y|X])^2 ]} \\
    &= \frac{n - d - 1}{d (\sigma_C^2 + \sigma_Y^2)} \left( \sigma_Y^2 \frac{\sigma_X^2}{\sigma_X^2 + \sigma_C^2} \frac{k}{n_2-k-1} + \sigma_C^2 \frac{d}{n_1 - d - 1} + \sigma_Y^2 \frac{1}{\sigma_X^2 + \sigma_C^2} \frac{1}{n_2 - k - 1} \sigma_C^2 \frac{d}{n_1 - d - 1} \right).
\end{align}
Taking the limit as $n$ goes to infinity and letting $n_1=n_2=n$ gives the desired result
\begin{align}
    \lim_{n \rightarrow \infty} \frac{\mathbb{E}[ (Y - \hat{Y}_{IB})^2 ] - \mathbb{E}[ (Y - \mathbb{E}[Y|X])^2 ]}{\mathbb{E}[ (Y - \hat{Y}_{SM})^2 ] - \mathbb{E}[ (Y - \mathbb{E}[Y|X])^2 ]}
    &= \frac{\sigma_Y^2}{\sigma_C^2+\sigma_Y^2} \frac{\sigma_X^2}{\sigma_X^2 + \sigma_C^2} \frac{k}{d}  + \frac{\sigma_C^2}{\sigma_C^2 + \sigma_Y^2} \\
    &\leq \frac{\frac{k}{d} \sigma_Y^2 + \sigma_C^2}{\sigma_Y^2 + \sigma_C^2}.
\end{align}
\end{proof}

\end{document}